\documentclass[runningheads]{llncs}
\usepackage{cite}
\usepackage{amsmath,amssymb,amsfonts}
\usepackage{algorithmic}
\usepackage{bm}
\usepackage{bbm}
\usepackage{mathtools}
\usepackage{graphicx}
\usepackage{textcomp}
\usepackage{xcolor}

\newcommand{\argmin}{\mathop{\rm arg~min}\limits}

\begin{document}
\title{Understanding Test-Time Augmentation}
\author{Masanari Kimura\inst{1}\orcidID{0000-0002-9953-3469}}
\authorrunning{Masanari Kimura}
\institute{Ridge-i Inc., Tokyo, Japan \\ \email{mkimura@ridge-i.com}\\ 
}
\maketitle              % typeset the header of the contribution
\begin{abstract}
Test-Time Augmentation (TTA) is a very powerful heuristic that takes advantage of data augmentation during testing to produce averaged output.
Despite the experimental effectiveness of TTA, there is insufficient discussion of its theoretical aspects.
In this paper, we aim to give theoretical guarantees for TTA and clarify its behavior.

\end{abstract}
\section{Introduction}
The effectiveness of machine learning has been reported for a great variety of tasks~\cite{mohri2018foundations,krizhevsky2012imagenet,kotsiantis2007supervised,arganda2017trainable,indurkhya2010handbook}.
However, satisfactory performance during testing is often not achieved due to the lack of training data or the complexity of the model.

One important concept to tackle such problems is data augmentation.
The basic idea of data augmentation is to increase the training data by transforming the input data in some way to generate new data that resembles the original instance.
Many data augmentations have been proposed~\cite{shorten2019survey,Zhong2020-bk,Park2019-xj,Kimura2020-oq}, ranging from simple ones, such as flipping input images~\cite{perez2017effectiveness,mikolajczyk2018data}, to more complex ones, such as leveraging Generative Adversarial Networks (GANs) to automatically generate data~\cite{frid2018synthetic,frid2018gan}.
In addition, there are several studies on automatic data augmentation in the framework of AutoML~\cite{Lim2019-qq,Hataya2020-vz}.

Another approach to improve the performance of machine learning models is ensemble learning~\cite{dietterich2002ensemble,polikar2012ensemble}.
Ensemble learning generates multiple models from a single training dataset and combines their outputs, hoping to outperform a single model.
The effectiveness of ensemble learning has also been reported in a number of domains~\cite{dong2020survey,fersini2014sentiment,Li2020-dk}.

Influenced by these approaches, a new paradigm called Test-Time Augmentation (TTA)~\cite{wang2019aleatoric,wang2018automatic,moshkov2020test} has been gaining attention in recent years.
TTA is a very powerful heuristic that takes advantage of data augmentation during testing to produce averaged output.
Despite the experimental effectiveness of TTA, there is insufficient discussion of its theoretical aspects.
In this paper, we aim to give theoretical guarantees for TTA and clarify its behavior.
Our contributions are summarized as follows:
\begin{itemize}
    \item We prove that the expected error of the TTA is less than or equal to the average error of an original model. Furthermore, under some assumptions, the expected error of the TTA is strictly less than the average error of an original model;
    \item We introduce the generalized version of the TTA, and the optimal weights of it are given by the closed-form;
    \item We prove that the error of the TTA depends on the ambiguity term.
\end{itemize}

\section{Preliminaries}
Here, we first introduce the notations and problem formulation.

\subsection{Problem formulation}
Let $\mathcal{X}\in\mathbb{R}^d$ be the $d$-dimensional input  space, $\mathcal{Y}\in\mathbb{R}$ be the output space, and $\mathcal{H} = \{h(\bm{x}; \bm{\theta}):\mathcal{X}\to\mathcal{Y}\ |\ \bm{\theta}\in\Theta \}$ be a hypothesis class, where $\Theta\subset\mathbb{R}^p$ is the $p$-dimensional parameter space.
In supervised learning, our goal is to obtain $h^*\in\mathcal{H} : \mathcal{X}\to\mathcal{Y}$ such that
\begin{equation}
    h^* = \argmin_{h\in\mathcal{H}}\mathcal{R}^\ell(h) = \argmin_{h\in\mathcal{H}}\mathbb{E}\Big[\ell(y, h(\bm{x};\bm{\theta}))\Big],
\end{equation}
where
\begin{equation}
    \mathcal{R}^\ell(h) \coloneqq \mathbb{E}\Big[\ell(y, h(\bm{x}; \bm{\theta}))\Big]
\end{equation}
is the expected error and $\ell:\mathcal{Y}\times\mathcal{Y}\to\mathbb{R}_+$ is some loss function.
Since we can not access $\mathcal{R}^\ell(h)$ directly, we try to approximate $\mathcal{R}^\ell(h)$ from the limited sample $S = \{(y_i, \bm{x}_i)\}^N_{i=1}$ of size $N\in\mathbb{N}$.
It is the ordinal empirical risk minimization (ERM) problem, and the minimizer of the empirical error $\hat{\mathcal{R}}^\ell_S\coloneqq\frac{1}{N}\sum^N_{i=1}\ell(y_i, h(\bm{x}_i))$ can be calculated as
\begin{equation}
    \hat{h} = \argmin_{h\in\mathcal{H}}\hat{\mathcal{R}}^\ell_S(h) = \argmin_{h\in\mathcal{H}}\frac{1}{N}\sum^N_{i=1}\ell(y_i, h(\bm{x}_i; \bm{\theta})).
\end{equation}

It is known that when the hypothesis class is complex (e.g., a class of neural networks), learning by ERM can lead to overlearning~\cite{hawkins2004problem}.
To tackle this problem, many approaches have been proposed, such as data augmentation~\cite{van2001art,perez2017effectiveness,zhang2017mixup} and ensemble learning~\cite{dietterich2002ensemble,polikar2012ensemble,dong2020survey}.
Among such methods, Test-Time Augmentation (TTA)~\cite{wang2019aleatoric,wang2018automatic,moshkov2020test} is an innovative paradigm that has attracted a great deal of attention in recent years.

\subsection{TTA: Test-Time Augmentation}
The TTA framework is generally described as follows:
let $\bm{x}\in\mathcal{X}$ be the new input variable at test time.
We now consider multiple data augmentations $\{\tilde{\bm{x}}_i \}^m_{i=1}$ for $\bm{x}$, where $\tilde{\bm{x}}_i\in\mathbb{R}^d$ is the $i$-th augmented data where $\bm{x}$ is transformed and $m$ is the number of strategies for data augmentation.
Finally, we compute the output $\tilde{y}$ for the original input $\bm{x}$ as $\tilde{y} = \sum^m_{i=1}h(\tilde{\bm{x}}_i)$.
Thus, intuitively, one would expect $\tilde{y}$ to be a better predictor than $y$.
TTA is a very powerful heuristic, and its effectiveness has been reported for many tasks~\cite{wang2019aleatoric,wang2018automatic,moshkov2020test,amiri2020two,Tian2020-yi}.
Despite its experimental usefulness, the theoretical analysis of TTA is insufficient.
In this paper, we aim to theoretically analyze the behavior of TTA.
In addition, at the end of the manuscript, we provide directions for future works~\cite{shanmugam2020and} on the theoretical analysis of TTA in light of the empirical observations given in existing studies.

\section{Theoretical results for the Test-Time Augmentation}
In this section, we give several theoretical results for the TTA procedure.
\subsection{Re-formalization of TTA}
First of all, we reformulate the TTA procedure as follows.
\begin{definition}{(Augmented input space)}
For the transformation class $\mathcal{G}$, we define the augmented input space $\bar{\mathcal{X}}$ as
\begin{equation}
    \bar{\mathcal{X}} \coloneqq \mathcal{X} \cup \Biggl(\bigcup^\infty_{i=1} g(\mathcal{X}; \bm{\xi}_i)\Biggr) = \mathcal{X}\cup\Biggl(\bigcup^\infty_{i=1}\bigcup^\infty_{j=1}g(\bm{x}_j;\bm{\xi}_i)\Biggr).
\end{equation}
\end{definition}
\begin{definition}{(TTA as the function composition)}
\label{def:tta}
Let $\mathcal{F}=\{f(\bm{x};\bm{\theta}_{\mathcal{F}})\ |\ \bm{\theta}_{\mathcal{F}}\in\Theta_{\mathcal{F}}\subset\Theta\}\subset\mathcal{H}$ be a subset of the hypothesis class and $\mathcal{G} = \{g(\bm{x};\bm{\xi}):\mathcal{X}\to\bar{\mathcal{X}}\ |\ \bm{\xi}\in\Xi \}$ be the transformation class.
We assume that $\{g_i = g(\bm{x};\bm{\xi}_i)\}^m_{i=1}$ is a set of the data augmentation strategies, and the TTA output $\tilde{y}$ for the input $\bm{x}$ is calculated as
\begin{equation}
    \tilde{y}(\bm{x}, \{\bm{\xi}^m_{i=1}\}) \coloneqq \sum^m_{i=1}f\circ g_i (\bm{x}) = \frac{1}{m}\sum^m_{i=1}f(g(\bm{x};\bm{\xi}_i); \bm{\theta}_{\mathcal{F}}).
\end{equation}
\end{definition}
From these definitions, we have the expected error for the TTA procedure as follows.
\begin{definition}{(Expected error with TTA)}
The empirical error $\mathcal{R}^{\ell,\mathcal{G}}$of the hypothesis $h\in\mathcal{H}$ obtained by the TTA with transformation class $\mathcal{G}$ is calculated as follows:
\begin{equation}
    \mathcal{R}^{\ell,\mathcal{G}}(h) \coloneqq \int_{\mathcal{X}\times\mathcal{Y}}\ell(y, \tilde{y}(\bm{x}, \{\bm{\xi}^m_{i=1}\}))p(\bm{x}, y)d\bm{x}dy.
\end{equation}
\end{definition}
The next question is, whether $\mathcal{R}^{\ell,\mathcal{G}}(h)$ is less than $\mathcal{R}^\ell$ or not.
In addition, if $\mathcal{R}^{\ell,\mathcal{G}}(h)$ is strictly less than $\mathcal{R}^\ell$, it is interesting to show the required assumptions.

\subsection{Upper bounds for the TTA}
Next we derive the upper bounds for the TTA.
For the sake of argument, we assume that $\ell(a, b) = (a - b)^2$ and we decompose the output of the hypothesis for $(\bm{x}, y)$ as follows:
\begin{equation}
    h(\bm{x};\bm{\theta}) = y + \epsilon(\bm{x},y;\bm{\theta}) \quad (\forall{h\in\mathcal{H}}).
\end{equation}
Then, the following theorem holds.
\begin{theorem}
Assume that $f\circ g \in \mathcal{H}$ for all $f\in\mathcal{F}$ and $g\in\mathcal{G}$, and $\mathcal{G}$ contains the identity transformation $g:\bm{x}\mapsto\bm{x}$.
Then, the expected error obtained by TTA is bounded from above by the average error of single hypothesises:
\begin{equation}
    \mathcal{R}^{\ell,\mathcal{G}}(h) \leq \bar{\mathcal{R}}^\ell(h) \coloneqq \mathbb{E}\Biggl[\frac{1}{m}\sum^m_{i=1}\ell(y, h(\bm{x};\bm{\theta}_i))\Biggr].
\end{equation}
\end{theorem}
\begin{proof}
From the definition, the ordinal expected average error is calculated as
\begin{align}
    \bar{\mathcal{R}}^\ell(h) &= \int_{\mathcal{X}\times\mathcal{Y}}\frac{1}{m}\sum^m_{i=1}(y - h(\bm{x};\bm{\theta}_i))^2 p(\bm{x}, y)d\bm{x}dy \\
    &= \int_{\mathcal{X}\times\mathcal{Y}}\frac{1}{m}\sum^m_{i=1}\epsilon(\bm{x},y;\bm{\theta}_i)^2p(\bm{x}, y)d\bm{x}dy. \label{eq:erm_error_mse}
\end{align}
On the other hand, the expected error of TTA is
\begin{align}
    \mathcal{R}^{\ell,\mathcal{G}}(h) &= \int_{\mathcal{X}\times\mathcal{Y}}\Biggl(y - \frac{1}{m}\sum^m_{i=1}f\circ g_i(\bm{x}) \Biggr)^2 p(\bm{x},y)d\bm{x}dy \nonumber \\
    &= \int_{\mathcal{X}\times\mathcal{Y}}\Biggl(\frac{1}{m}\sum^m_{i=1}\Big(y - f\circ g_i(\bm{x})\Big)\Biggr)^2 p(\bm{x}, y)d\bm{x}dy \\
    &= \int_{\mathcal{X}\times\mathcal{Y}}\Biggl(\frac{1}{m}\sum^m_{i=1}\epsilon(\bm{x},y;\bm{\theta}_i)\Biggr)^2 p(\bm{x}, y)d\bm{x}dy. \label{eq:tta_error_mse}
\end{align}
Then, from Eq.~\eqref{eq:erm_error_mse} and \eqref{eq:tta_error_mse}, we have the proof of the theorem.
\end{proof}
By making further assumptions, we also have the following theorem.
\begin{theorem}
Assume that $f\circ g \in \mathcal{H}$ for all $f\in\mathcal{F}$ and $g\in\mathcal{G}$, and $\mathcal{G}$ contains the identity transformation $g:\bm{x}\mapsto\bm{x}$.
Assume also that each $\epsilon$ has mean zero and is uncorrelated with each other:
\begin{align}
    &\int_{\mathcal{X}\times\mathcal{Y}}\epsilon(\bm{x}, y; \bm{\theta}_i)p(\bm{x}, y)d\bm{x}dy = 0 \quad (\forall i\in\{1,\dots,m\}), \label{eq:mean_zero} \\
    &\int_{\mathcal{X}\times\mathcal{Y}}\epsilon(\bm{x}, y; \bm{\theta}_i)\epsilon(\bm{x}, y; \bm{\theta}_j)p(\bm{x}, y)d\bm{x}dy = 0 \quad (i \neq j). \label{eq:uncorrelated}
\end{align}
In this case, the following relationship holds
\begin{equation}
    \mathcal{R}^{\ell,\mathcal{G}}(h) = \frac{1}{m}\bar{\mathcal{R}}^\ell(h) < \bar{\mathcal{R}}^\ell(h) .
\end{equation}
\end{theorem}
\begin{proof}
From the assumptions~\eqref{eq:mean_zero}, ~\eqref{eq:uncorrelated} and Eq.~\eqref{eq:erm_error_mse} and \eqref{eq:tta_error_mse}, the proof of the theorem can be obtained immediately.
\end{proof}

\subsection{Weighted averaging for the TTA}
We consider the generalization of TTA as follows.
\begin{definition}{(Weighted averaging for the TTA)}
Let $\mathcal{F}=\{f(\bm{x};\bm{\theta}_{\mathcal{F}})\ |\ \bm{\theta}_{\mathcal{F}}\in\Theta_{\mathcal{F}}\subset\Theta\}\subset\mathcal{H}$ be a subset of the hypothesis class and $\mathcal{G} = \{g(\bm{x};\bm{\xi}):\mathcal{X}\to\bar{\mathcal{X}}\ |\ \bm{\xi}\in\Xi \}$ be the transformation class.
We assume that $\{g_i = g(\bm{x};\bm{\xi}_i)\}^m_{i=1}$ is a set of the data augmentation strategies, and the TTA output $\tilde{y}$ for the input $\bm{x}$ is calculated as
\begin{equation}
    \tilde{y}_w(\bm{x}, \{\bm{\xi}^m_{i=1}\}) \coloneqq \sum^m_{i=1}w_i f\circ g_i (\bm{x}) = \sum^m_{i=1}w(\bm{\xi}_i) f(g(\bm{x};\bm{\xi}_i); \bm{\theta}_{\mathcal{F}}), \label{eq:wtta}
\end{equation}
where $w_i = w(\bm{\xi}_i): \Xi\to\mathbb{R}_+$ is the weighting function:
\begin{equation}
    w_i \geq 0\ (\forall i\in\{1,\dots,m\}), \quad \sum^m_{i=1}w_i = 1 \label{eq:weights_condition}
\end{equation}
\end{definition}

Then, we can obtain the expected error of Eq.~\eqref{eq:wtta} as follows.

\begin{proposition}
The expected error of the weighted TTA is
\begin{equation}
    \label{prop:expected_error_tta}
    \mathcal{R}^{\ell,\mathcal{G},w}(h) = \sum^m_{i=1}\sum^m_{j=i}w_iw_j\Gamma_{ij},
\end{equation}
where
\begin{equation}
    \Gamma_{ij} = \int_{\mathcal{X}\times\mathcal{Y}}\Big(y - f\circ g_i(\bm{x})\Big)\Big(y - f\circ g_j(\bm{x}) \Big)p(\bm{x},y)d\bm{x}dy.
\end{equation}
\end{proposition}
\begin{proof}
We can calculate as
\begin{align}
    \mathcal{R}^{\ell, \mathcal{G}, w}(h) &= \int_{\mathcal{X}\times\mathcal{Y}}\Biggl(y - \sum^m_{i=1}w_i f\circ g_i(\bm{x})\Biggr)^2 p(\bm{x}, y)d\bm{x}dy \nonumber \\
    &= \int_{\mathcal{X}\times\mathcal{Y}}\Biggl(y - \sum^m_{i=1}w_i f\circ g_i(\bm{x})\Biggr)\Biggl(y - \sum^m_{j=1}w_j f\circ g_j(\bm{x})\Biggr)p(\bm{x},y)d\bm{x}dy \nonumber \\
    &= \sum^m_{i=1}\sum^m_{j=i}w_iw_j\Gamma_{ij}. \label{eq:generalized_tta_error}
\end{align}
\end{proof}
Proposition~\ref{prop:expected_error_tta} implies that the expected error of the weighted TTA is highly depending on the correlations of $\{g_1,\dots,g_m\}^m_{i=1}$.

\begin{theorem}{(Optimal weights for the weighted TTA)}
\label{thm:optimal_weights}
We can obtain the optimal weights $\bm{w} = \{w_i,\dots,w_j\}$ for the weighted TTA as follows:
\begin{equation}
    w_i = \frac{\sum^m_{j=1}\Gamma^{-1}_{ij}}{\sum^m_{k=1}\sum^m_{j=1}\Gamma^{-1}_{kj}},
\end{equation}
where $\Gamma^{-1}_{ij}$ is the $(i,j)$-element of the inverse matrix of $(\Gamma_{ij})$.
\begin{proof}
The optimal weights can be obtained by solving
\begin{equation}
    \bm{w} = \argmin_{\bm{w}}\sum^m_{i=1}\sum^m_{j=1}w_iw_j\Gamma_{ij}.
\end{equation}
Then, from the method of Lagrange multiplier,
\begin{align}
    \frac{\partial}{\partial w_k}\Biggl\{\sum^m_{i=1}\sum^m_{j=1}w_iw_j\Gamma_{ij} - 2\lambda\Big(\sum^m_{i=1}w_i - 1\Big)\Biggr\} &= 0 \\
    2\sum^m_{j=1}w_k\Gamma_{kj} &= 2\lambda \\
    \sum^m_{j=1}w_k\Gamma_{kj} &= \lambda.
\end{align}
From the condition~\eqref{eq:weights_condition}, since $\sum^m_{i=1}w_i = 1$ and then, we have
\begin{equation}
    w_i = \frac{\sum^m_{j=1}\Gamma^{-1}_{ij}}{\sum^m_{k=1}\sum^m_{j=1}\Gamma^{-1}_{kj}}.
\end{equation}
\end{proof}
\end{theorem}

From Theorem~\ref{thm:optimal_weights}, we obtain a closed-form expression for the optimal weights of the weighted TTA.
Furthermore, we see that this solution requires an invertible correlation matrix $\Gamma$.
However, in TTA we consider the set of $\{f\circ g_i\}^m_{i=1}$, and all elements depend on $f\in\mathcal{F}$ in common.
This means that the correlations among $\{f\circ g_i\}^m_{i=1}$ will be very high, and such correlation matrix is generally known to be singular or ill-conditioned.

\subsection{Existence of the unnecessary transformation functions}
To simplify the discussion, we assume that all weights are equal.
Then, from Eq.~\eqref{eq:generalized_tta_error}, we have
\begin{equation}
    \mathcal{R}^{\ell, \mathcal{G}, w}(h) = \sum^m_{i=1}\sum^m_{j=1}\Gamma_{ij} / m^2.
\end{equation}
If we remove $g_k$ from $\{g_1,\dots,g_m\}$, the error $\tilde{\mathcal{R}}^{\ell, \mathcal{G}, w}(h)$ is recomputed as follows.
\begin{equation}
    \tilde{\mathcal{R}}^{\ell, \mathcal{G}, w}(h) = \sum^m_{\substack{i=1\\i\neq k}}\sum^m_{\substack{j=1\\j\neq k}}\Gamma_{ij} / (m-1)^2.
\end{equation}
Here we consider how the error of the TTA changes when we remove the $k$-th data augmentation.
If we assume that $\mathcal{R}^{\ell,\mathcal{G},w}(h)$ is greater than or equal to $\tilde{\mathcal{R}}^{\ell, \mathcal{G}, w}(h)$, then
\begin{equation}
    (2m-1)\sum^m_{i=1}\sum^m_{j=1}\Gamma_{ij} \leq 2m^2\sum^m_{\substack{i=1 \\ i \neq k}}\Gamma_{ik} + m^2\Gamma_{kk}. \label{eq:tta_pruning}
\end{equation}
From the above equation, we can see that the group of data augmentations with very high correlation is redundant, except for some of them.
Fig.~\ref{fig:tta_pruning} shows an example of a numerical experiment to get the probability that Eq.~\eqref{eq:tta_pruning} holds.
In this numerical experiment, we generated a sequence of random values with the specified correlation to obtain a pseudo $(\Gamma_{ij})$, and calculated the probability that Eq.~\eqref{eq:tta_pruning} holds out of $100$ trials.
From this plot, we can see that $(\Gamma_{ij})$ with high correlation is likely to have redundancy.
In the following, we introduce ambiguity as another measure of redundancy and show that this measure is highly related to the error of TTA.

\begin{figure}[t]
    \centering
    \includegraphics[width=1.0\linewidth]{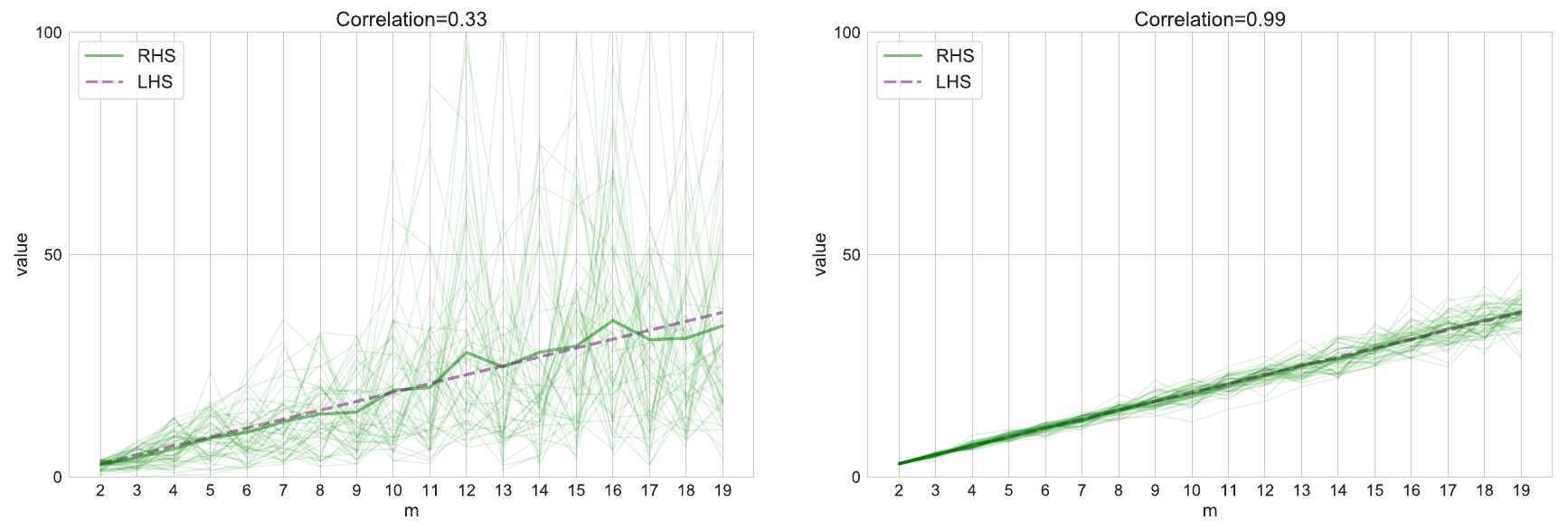}
    \caption{$(2m-1)\sum^m_{i=1}\sum^m_{j=1}\Gamma_{ij}=$LHS vs RHS$=2m^2\sum^m_{i \neq k}\Gamma_{ik} + m^2\Gamma_{kk}$ (Eq.~\eqref{eq:tta_pruning}).
    When the correlation is $0.33$, the numerical calculation yields $\Pr(RHS \geq LHS) \approx 0.38$.
    On the other hand, when the correlation is $0.99$, we yields $\Pr(RHS \geq LHS) \approx 0.49$.}
    \label{fig:tta_pruning}
\end{figure}

\subsection{Error decomposition for the TTA}
Knowing what elements the error can be broken down into is one important way to understand the behavior of TTA.
For this purpose, we introduce the following notion of ambiguity.
\begin{definition}{(Ambiguity of the hypothesis set~\cite{krogh1995validation})}
\label{def:ambiguity}
For some $\bm{x}\in\mathcal{X}$, the ambiguity $\varsigma(h_i | \bm{x})$ of the hypothesis set $h = \{h_i\}^m$ is defined as
\begin{equation}
    \varsigma(h_i | \bm{x}) \coloneqq \Biggl(h_i(\bm{x}) - \sum^m_{i=1}w_i h_i(\bm{x})\Biggr)^2 \quad (\forall i\in\{1,\dots,m\}).
\end{equation}
\end{definition}
Let $\bar{\varsigma}(h|\bm{x})$ be the average ambiguity: $\bar{\varsigma}(h|\bm{x})=\sum^m_{i=1}w_i\varsigma(h_i | \bm{x})$.
From Definition~\ref{def:ambiguity}, the ambiguity term can be regarded as a measure of the discrepancy between individual hypotheses for input $\bm{x}$.
Then, we have
\begin{align}
    \bar{\varsigma}(h|\bm{x}) = \sum^m_{i=1}w_i (y - f\circ g_i(\bm{x}))^2 - (y - \sum^m_{i=1}w_i f\circ g_i(\bm{x}))^2. \label{eq:ambiguity_error}
\end{align}
Since Eq.~\eqref{eq:ambiguity_error} holds for all $\bm{x}\in\mathcal{X}$,
\begin{align}
    &\sum^m_{i=1}w_i \int_{\mathcal{X}\times\mathcal{Y}}\varsigma(h_i|\bm{x})p(\bm{x}, y)d\bm{x}dy \nonumber \\
    & = \sum^m_{i=1}w_i\int (y - f\circ g_i(\bm{x}))^2 p(\bm{x},y)d\bm{x}dy - \int \Big(y - \sum^m_{i=1}w_i f\circ g_i(\bm{x})\Big)^2p(\bm{x},y)d\bm{x}dy. \nonumber
\end{align}
Let
\begin{align}
    err(f\circ g_i) = \mathbb{E}\Big[(y - f\circ g_i(\bm{x}))^2\Big] = \int (y - f\circ g_i(\bm{x}))^2p(\bm{x})d\bm{x}dy,
\end{align}
and
\begin{align}
    \varsigma(f\circ g_i) = \mathbb{E}\Big[\varsigma(f\circ g_i | \bm{x})\Big] = \int \varsigma(f\circ g_i | \bm{x})p(\bm{x})d\bm{x}.
\end{align}
Then, we have
\begin{equation}
    \mathcal{R}^{\ell,\mathcal{G},w}(h) = \sum^m_{i=1}w_i\cdot err(f\circ g_i) - \sum^m_{i=1}w_i\cdot \varsigma(f\circ g_i)),
\end{equation}
where the first term corresponds to the error, and the second term corresponds to the ambiguity.
From this equation, it can be seen that TTA yields significant benefits when each $f\circ g_i$ is more accurate and more diverse than the other.

To summarize, we have the following proposition.
\begin{proposition}
The error of the TTA can be decomposed as
\begin{equation}
    \mathcal{R}^{\ell,\mathcal{G},w}(h) = \Big[ \text{errors of $f\circ g_i$}\Big] + \Big[ \text{ambiguities of $f\circ g_i$}\Big].
\end{equation}
\end{proposition}

\subsection{Statistical consistency}
Finally, we discuss the statistical consistency for the TTA procedure.

\begin{definition}
The ERM is the strictly consistent if for any non-empty subset $\mathcal{H}(c) = \{h\in\mathcal{H}: \mathcal{R}^\ell(h) \geq c\}$ with $c\in(-\infty, +\infty)$ the following convergence holds:
\begin{equation}
    \inf_{h\in\mathcal{H}(c)}\hat{\mathcal{R}}^\ell_{\mathcal{S}}(h) \overset{p}{\to} \inf_{h\in\mathcal{H}(c)}\mathcal{R}^\ell(h) \quad (N\to\infty).
\end{equation}
\end{definition}
Necessary and sufficient conditions for strict consistency are provided by the following theorem~\cite{vapnik1999overview,vapnik2013nature}.
\begin{theorem}
If two real constants $a\in\mathbb{R}$ and $A\in\mathbb{R}$ can be found such that for every $h\in\mathcal{H}$ the inequalities $a \leq \mathcal{R}^\ell(h) \leq A$ hold, then the following two statements are equivalent:
\begin{enumerate}
    \item The empirical risk minimization is strictly consistent on the set of functions $\{\ell(y,h(\bm{x}))\ |\ h\in\mathcal{H}\}$.
    \item  The uniform one-sided convergence of the mean to their expectation takes place over the set of functions $\{\ell(y,h(\bm{x}))\ |\ h\in\mathcal{H}\}$, i.e.,
    \begin{equation}
        \lim_{N\to\infty}\Pr\Big[\sup_{h\in\mathcal{H}}\Big\{\mathcal{R}^\ell(h) - \hat{\mathcal{R}}^\ell_{\mathcal{S}}(h)\Big\} > \epsilon \Big] = 0 \quad (\forall \epsilon > 0).
    \end{equation}
\end{enumerate}
\end{theorem}
Using these concepts, we can derive the following lemma.
\begin{lemma}
\label{lem:consistency}
The empirical risk $\frac{1}{m}\sum^m_{i=1}\hat{\mathcal{R}}^{\ell,\mathcal{G}}_{\mathcal{S}}(h)$ obtained by ERM with data augmentations $\{g_1,\dots,g_m\}^m_{i=1}$ is the consistent estimator of $\mathbb{E}_{\bar{\mathcal{X}}\times\mathcal{Y}}\Big[\ell(y, f(\bm{x}))\Big]$, i.e.,
\begin{equation}
    \inf_{h\in\mathcal{H}(c)}\hat{\mathcal{R}}^{\ell,\mathcal{G}}_{\mathcal{S}}(h) \overset{p}{\to} \mathbb{E}_{\bar{\mathcal{X}}\times\mathcal{Y}}\Big[\ell(y, f(\bm{x}))\Big] \quad (N\to\infty).
\end{equation}
\end{lemma}
\begin{proof}
Let $\bar{\mathcal{X}}$ be the augmented input space with $\{g_1,\dots,g_m\}^m_{i=1}$.
Then, we have
\begin{align}
    \mathbb{E}_{\mathcal{X}\times\mathcal{Y}}\Big[\hat{\mathcal{R}}^{\ell,\mathcal{G}}_{\mathcal{S}}(h)\Big] &= \int_{\mathcal{X}\times\mathcal{Y}}\Biggl\{\frac{1}{N}\sum^N_{i=1}\frac{1}{m}\sum^m_{j=1}\ell(y, f\circ g_j(\bm{x})) \Biggr\} p(\bm{x},y)d\bm{x}dy \\
    &= \int_{\mathcal{X}\times\mathcal{Y}}\Biggl\{\frac{1}{Nm}\sum^N_{i=1}\sum^m_{j=1}\ell(y, f\circ g_j(\bm{x})) \Biggr\} p(\bm{x},y)d\bm{x}dy \\
    &= \int_{\bar{\mathcal{X}}\times\mathcal{Y}}\Biggl\{\frac{1}{Nm}\sum^{Nm}_{i=1}\ell(y, f(\bm{x}))\Biggr\}p(\bm{x}, y)d\bm{x}dy \\
    &= \int_{\bar{\mathcal{X}}\times\mathcal{Y}} \ell(y, f(\bm{x})) p(\bm{x}, y)d\bm{x}dy = \mathbb{E}_{\bar{\mathcal{X}}\times\mathcal{Y}}\Big[\ell(y, f(\bm{x}))\Big].
\end{align}
\end{proof}

From Lemma~\ref{lem:consistency}, we can confirm that the ERM with data augmentation is also minimizing the TTA error.
This means that the data augmentation strategies used in TTA should also be used during training.

\section{Related works}
Although there is no existing research that discusses the theoretical analysis of the TTA, there are some papers that experimentally investigate the behavior of the TTA~\cite{shanmugam2020and}.
In those papers, the following results are reported:
\begin{itemize}
    \item the benefit of TTA depends upon the model’s lack of invariance to the given Test-Time Augmentations;
    \item as the training sample size increases, the benefit of TTA decreases;
    \item when TTA was applied to two datasets, ImageNet~\cite{krizhevsky2012imagenet} and Flowers-102~\cite{Nilsback08}, the performance improvement on the Flowers-102 dataset was small.
\end{itemize}

Because of the simplicity of the concept, several variants of TTA have also been proposed~\cite{kim2020learning,pmlr-v124-lyzhov20a,DBLP:journals/corr/abs-2105-06183}.
It is also a critical research direction to consider whether a theoretical analysis of these variants is possible using the same procedure as discussed in this paper.

\section{Conclusion and Discussion}
In this paper, we theoretically investigate the behavior of TTA.
Our discussion shows that TTA has several theoretically desirable properties.
Furthermore, we showed that the error of TTA depends on the ambiguity of the output.

\begin{figure}[t]
    \centering
    \includegraphics[width=0.95\linewidth]{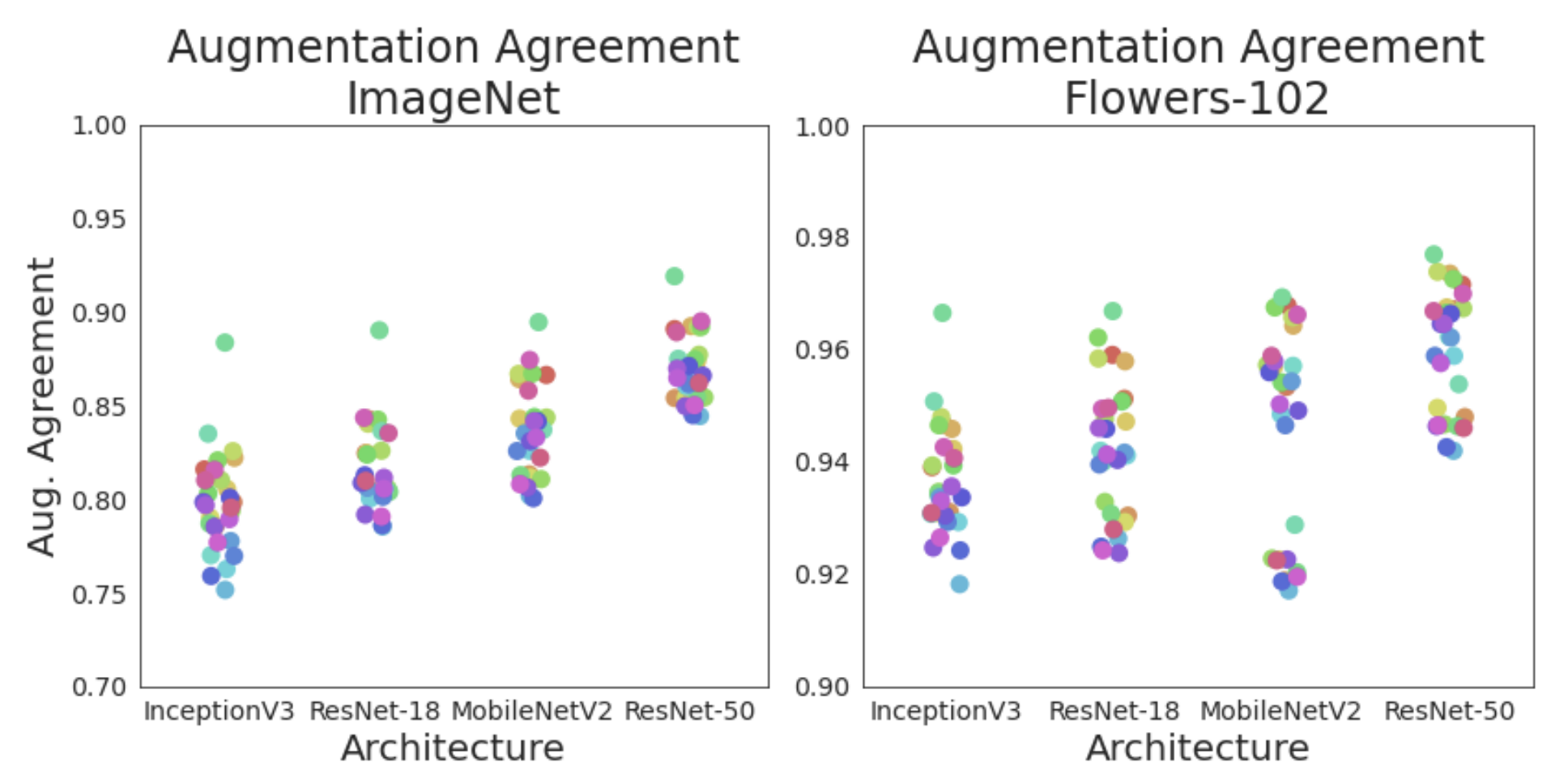}
    \caption{Architectures that benefit least from standard
TTA are also the least sensitive to the augmentations.
Note that this figure is created by \cite{shanmugam2020and}, and see their paper for more details.}
    \label{fig:tta_agreement}
\end{figure}

\subsection{Future works}
In the previous work, some empirical observations are reported~\cite{shanmugam2020and}.
The future of research is to construct a theory consistent with these observations.
\begin{itemize}
    \item When TTA was applied to two datasets, ImageNet~\cite{krizhevsky2012imagenet} and Flowers-102~\cite{Nilsback08}, the performance improvement on the Flowers-102 dataset was small.
    This may be because the instances in Flower-102 are more similar to each other than in the case of ImageNet, and thus are less likely to benefit from TTA.
    Figure~\ref{fig:tta_agreement} shows the relationship between the model architectures and the TTA ambiguity for each dataset~\cite{shanmugam2020and}.
    This can be seen as an analogous consideration to our discussion of ambiguity.
    \item The benefit of TTA varies depending on the model.
    Complex models have a smaller performance improvement from TTA than simple models.
    It is expected that the derivation of the generalization bound considering the complexity of the model such as VC-dimension and Rademacher complexity~\cite{alzubi2018machine,mohri2018foundations} will provide theoretical support for this experiment.
    \item The effect of TTA is larger in the case of the small amount of data.
    It is expected to be theorized by deriving inequalities depending on the sample size.
\end{itemize}

\bibliographystyle{splncs04}
% \bibliography{references}

\begin{thebibliography}{10}
\providecommand{\url}[1]{\texttt{#1}}
\providecommand{\urlprefix}{URL }
\providecommand{\doi}[1]{https://doi.org/#1}

\bibitem{alzubi2018machine}
Alzubi, J., Nayyar, A., Kumar, A.: Machine learning from theory to algorithms:
  an overview. In: Journal of physics: conference series. vol.~1142, p. 012012.
  IOP Publishing (2018)

\bibitem{amiri2020two}
Amiri, M., Brooks, R., Behboodi, B., Rivaz, H.: Two-stage ultrasound image
  segmentation using u-net and test time augmentation. International journal of
  computer assisted radiology and surgery  \textbf{15}(6),  981--988 (2020)

\bibitem{arganda2017trainable}
Arganda-Carreras, I., Kaynig, V., Rueden, C., Eliceiri, K.W., Schindelin, J.,
  Cardona, A., Sebastian~Seung, H.: Trainable weka segmentation: a machine
  learning tool for microscopy pixel classification. Bioinformatics
  \textbf{33}(15),  2424--2426 (2017)

\bibitem{dietterich2002ensemble}
Dietterich, T.G., et~al.: Ensemble learning. The handbook of brain theory and
  neural networks  \textbf{2},  110--125 (2002)

\bibitem{dong2020survey}
Dong, X., Yu, Z., Cao, W., Shi, Y., Ma, Q.: A survey on ensemble learning.
  Frontiers of Computer Science  \textbf{14}(2),  241--258 (2020)

\bibitem{fersini2014sentiment}
Fersini, E., Messina, E., Pozzi, F.A.: Sentiment analysis: Bayesian ensemble
  learning. Decision support systems  \textbf{68},  26--38 (2014)

\bibitem{frid2018gan}
Frid-Adar, M., Diamant, I., Klang, E., Amitai, M., Goldberger, J., Greenspan,
  H.: Gan-based synthetic medical image augmentation for increased cnn
  performance in liver lesion classification. Neurocomputing  \textbf{321},
  321--331 (2018)

\bibitem{frid2018synthetic}
Frid-Adar, M., Klang, E., Amitai, M., Goldberger, J., Greenspan, H.: Synthetic
  data augmentation using gan for improved liver lesion classification. In:
  2018 IEEE 15th international symposium on biomedical imaging (ISBI 2018). pp.
  289--293. IEEE (2018)

\bibitem{Hataya2020-vz}
Hataya, R., Zdenek, J., Yoshizoe, K., Nakayama, H.: Faster {AutoAugment}:
  Learning augmentation strategies using backpropagation. In: Computer Vision
  -- {ECCV} 2020. pp. 1--16. Springer International Publishing (2020)

\bibitem{hawkins2004problem}
Hawkins, D.M.: The problem of overfitting. Journal of chemical information and
  computer sciences  \textbf{44}(1),  1--12 (2004)

\bibitem{indurkhya2010handbook}
Indurkhya, N., Damerau, F.J.: Handbook of natural language processing, vol.~2.
  CRC Press (2010)

\bibitem{kim2020learning}
Kim, I., Kim, Y., Kim, S.: Learning loss for test-time augmentation. arXiv
  preprint arXiv:2010.11422  (2020)

\bibitem{Kimura2020-oq}
Kimura, M.: Why mixup improves the model performance  (Jun 2020)

\bibitem{kotsiantis2007supervised}
Kotsiantis, S.B., Zaharakis, I., Pintelas, P.: Supervised machine learning: A
  review of classification techniques. Emerging artificial intelligence
  applications in computer engineering  \textbf{160}(1),  3--24 (2007)

\bibitem{krizhevsky2012imagenet}
Krizhevsky, A., Sutskever, I., Hinton, G.E.: Imagenet classification with deep
  convolutional neural networks. Advances in neural information processing
  systems  \textbf{25},  1097--1105 (2012)

\bibitem{krogh1995validation}
Krogh, A., Vedelsby, J.: Validation, and active learning. Advances in neural
  information processing systems 7  \textbf{7}, ~231 (1995)

\bibitem{Li2020-dk}
Li, Y., Hu, G., Wang, Y., Hospedales, T., Robertson, N.M., Yang, Y.:
  Differentiable automatic data augmentation. In: Computer Vision -- {ECCV}
  2020. pp. 580--595. Springer International Publishing (2020)

\bibitem{Lim2019-qq}
Lim, S., Kim, I., Kim, T., Kim, C., Kim, S.: Fast {AutoAugment}  (May 2019)

\bibitem{pmlr-v124-lyzhov20a}
Lyzhov, A., Molchanova, Y., Ashukha, A., Molchanov, D., Vetrov, D.: Greedy
  policy search: A simple baseline for learnable test-time augmentation. In:
  Peters, J., Sontag, D. (eds.) Proceedings of the 36th Conference on
  Uncertainty in Artificial Intelligence (UAI). Proceedings of Machine Learning
  Research, vol.~124, pp. 1308--1317. PMLR (03--06 Aug 2020),
  \url{http://proceedings.mlr.press/v124/lyzhov20a.html}

\bibitem{mikolajczyk2018data}
Miko{\l}ajczyk, A., Grochowski, M.: Data augmentation for improving deep
  learning in image classification problem. In: 2018 international
  interdisciplinary PhD workshop (IIPhDW). pp. 117--122. IEEE (2018)

\bibitem{DBLP:journals/corr/abs-2105-06183}
Mocerino, L., Rizzo, R.G., Peluso, V., Calimera, A., Macii, E.: Adaptive
  test-time augmentation for low-power {CPU}. CoRR  \textbf{abs/2105.06183}
  (2021), \url{https://arxiv.org/abs/2105.06183}

\bibitem{mohri2018foundations}
Mohri, M., Rostamizadeh, A., Talwalkar, A.: Foundations of machine learning.
  MIT press (2018)

\bibitem{moshkov2020test}
Moshkov, N., Mathe, B., Kertesz-Farkas, A., Hollandi, R., Horvath, P.:
  Test-time augmentation for deep learning-based cell segmentation on
  microscopy images. Scientific reports  \textbf{10}(1), ~1--7 (2020)

\bibitem{Nilsback08}
Nilsback, M.E., Zisserman, A.: Automated flower classification over a large
  number of classes. In: Indian Conference on Computer Vision, Graphics and
  Image Processing (Dec 2008)

\bibitem{Park2019-xj}
Park, D.S., Chan, W., Zhang, Y., Chiu, C.C., Zoph, B., Cubuk, E.D., Le, Q.V.:
  {SpecAugment}: A simple data augmentation method for automatic speech
  recognition  (Apr 2019)

\bibitem{perez2017effectiveness}
Perez, L., Wang, J.: The effectiveness of data augmentation in image
  classification using deep learning. arXiv preprint arXiv:1712.04621  (2017)

\bibitem{polikar2012ensemble}
Polikar, R.: Ensemble learning. In: Ensemble machine learning, pp. 1--34.
  Springer (2012)

\bibitem{shanmugam2020and}
Shanmugam, D., Blalock, D., Balakrishnan, G., Guttag, J.: When and why
  test-time augmentation works. arXiv preprint arXiv:2011.11156  (2020)

\bibitem{shorten2019survey}
Shorten, C., Khoshgoftaar, T.M.: A survey on image data augmentation for deep
  learning. Journal of Big Data  \textbf{6}(1),  1--48 (2019)

\bibitem{Tian2020-yi}
Tian, K., Lin, C., Sun, M., Zhou, L., Yan, J., Ouyang, W.: Improving
  {Auto-Augment} via {Augmentation-Wise} weight sharing  (Sep 2020)

\bibitem{van2001art}
Van~Dyk, D.A., Meng, X.L.: The art of data augmentation. Journal of
  Computational and Graphical Statistics  \textbf{10}(1),  1--50 (2001)

\bibitem{vapnik2013nature}
Vapnik, V.: The nature of statistical learning theory. Springer science \&
  business media (2013)

\bibitem{vapnik1999overview}
Vapnik, V.N.: An overview of statistical learning theory. IEEE transactions on
  neural networks  \textbf{10}(5),  988--999 (1999)

\bibitem{wang2019aleatoric}
Wang, G., Li, W., Aertsen, M., Deprest, J., Ourselin, S., Vercauteren, T.:
  Aleatoric uncertainty estimation with test-time augmentation for medical
  image segmentation with convolutional neural networks. Neurocomputing
  \textbf{338},  34--45 (2019)

\bibitem{wang2018automatic}
Wang, G., Li, W., Ourselin, S., Vercauteren, T.: Automatic brain tumor
  segmentation using convolutional neural networks with test-time augmentation.
  In: International MICCAI Brainlesion Workshop. pp. 61--72. Springer (2018)

\bibitem{zhang2017mixup}
Zhang, H., Cisse, M., Dauphin, Y.N., Lopez-Paz, D.: mixup: Beyond empirical
  risk minimization. arXiv preprint arXiv:1710.09412  (2017)

\bibitem{Zhong2020-bk}
Zhong, Z., Zheng, L., Kang, G., Li, S., Yang, Y.: Random erasing data
  augmentation. AAAI  \textbf{34}(07),  13001--13008 (Apr 2020)

\end{thebibliography}

%
\end{document}